\documentclass[letterpaper]{article} 
\usepackage{aaai2027} 
\usepackage[hyphens]{url} 
\usepackage{graphicx} 
\usepackage{natbib} 
\usepackage{caption} 
\usepackage{microtype}
\usepackage{amsmath}
\usepackage{amssymb}
\usepackage{amsthm}
\usepackage{booktabs}
\usepackage{array}
\usepackage{algorithm}
\usepackage{algorithmic}

\floatstyle{ruled}
\restylefloat{algorithm}

\newcolumntype{P}[1]{>{\raggedright\arraybackslash}p{#1}}
\nocopyright
\title{HALO: Heterogeneous Admission through Localized Obligations for Safe Agentic Execution}

\author{
Taewoo Park,
Kyeonghyun Yoo,
Kiseok Kim,
Seunghyun Yoo,
and Hwangnam Kim\thanks{Corresponding author: hnkim@korea.ac.kr}
}

\affiliations{
Korea University\\
Seoul, Republic of Korea\\
\texttt{\{taewoopark, seven1705, kisuk528, seunghyunyoo, hnkim\}@korea.ac.kr}
}

\newcommand{\gazebo}{Gazebo}
\newcommand{\pxfour}{PX4}
\newcommand{\tablefont}{\footnotesize}
\newtheorem{lemma}{Lemma}
\newtheorem{corollary}{Corollary}
\newtheorem{proposition}{Proposition}
\newtheorem{theorem}{Theorem}
\newcommand{\method}{HALO}
\begin{document}
\maketitle

\begin{center}
\small\textit{Preprint. This manuscript has not yet been peer reviewed.}
\end{center}
\setcounter{dbltopnumber}{4}
\setcounter{topnumber}{4}
\setcounter{totalnumber}{6}
\renewcommand{\topfraction}{0.97}
\renewcommand{\dbltopfraction}{0.97}
\renewcommand{\textfraction}{0.02}
\renewcommand{\floatpagefraction}{0.90}
\renewcommand{\dblfloatpagefraction}{0.90}

\begin{abstract}
Recent agentic AI systems may return a heterogeneous response containing
notices, requests, handoffs, and actions. Conditions can change before external use, so components from the same response need not remain supported together. Rejecting the whole response discards useful components, whereas checking components independently can leave a dependent
without its prerequisite. We present Heterogeneous Admission with Localized Obligations (HALO), a runtime
protocol that preserves supported components whose declared prerequisites also
remain supported, rechecks each exact action before dispatch, and allows
blocked actions to be replaced only by fresh candidates. HALO matched all 96 admission expectations and passed all 20 protocol tests. In structured-response replay, it retained 248/248 supported components, including
128/128 unaffected by unrelated changes, while a whole-response policy retained
0/248. Across ten cold-start PX4/Gazebo sessions, HALO blocked every tested stale route, observed no matching stale setpoint, and completed all fresh recoveries.
\end{abstract}

\section{Introduction}

Large language model (LLM) responses are no longer merely textual. They can now include tool calls, software modifications, and commands for external systems. LLM-based systems that combine language models with these external actions are called agentic AI systems \cite{react2023,toolformer2023,injecagent2024}.

This change creates a new safety problem. A wrong answer may mislead a user, but a wrong action can corrupt software, damage equipment, disrupt critical operations, or put people at risk. Safe execution therefore becomes a responsibility of the runtime, the software layer that controls how an agent interacts with external systems.

An LLM needs time to generate a response. While a response is generated or waits for external use, the external system may continue to change. Sensor data may become outdated, authorization may expire, software state may change, or the physical environment may move. An action may therefore be outdated when generation finishes or before it reaches the controlled interface.

This issue will become more serious as future agents use more tools, produce more complex responses, and coordinate more actions. We call this problem \emph{runtime support drift}. It occurs when the conditions supporting an output change before the output is returned or used.

Runtime support drift is harder to handle when one response contains several outputs. Consider an unmanned aerial vehicle (UAV) agent that returns an operator notice, a sensor report, a flight action, and a status message that depends on that action. If the map or traffic information supporting the flight action changes, the action and its dependent status should be blocked. However, the notice and sensor report may still remain valid.

Simple approaches do not handle this situation well. Rejecting the entire response discards valid outputs. Checking every output independently may leave a dependent output without what it requires. Checking only once may also miss changes before the action reaches the external system.

We present Heterogeneous Admission with Localized Obligations (HALO), a runtime protocol for this problem. HALO treats each output as a separate \emph{component}. It preserves valid components, removes invalid ones, and rechecks each action immediately before \emph{dispatch} to the external system.

When a blocked action can be corrected, HALO records what must be updated and which component may receive a replacement. The old action is not reactivated. A fresh candidate must return with current information and new authorization, then pass the runtime checks again.

We evaluate HALO in UAV operation, where sensor data, geometry, authorization, and vehicle state can change quickly. This setting makes runtime support drift observable at the PX4 setpoint interface. The same runtime process can apply to other domains when an agent returns multiple components, current support can be checked, and actions pass through a runtime-controlled interface.

Our contributions are:

\begin{itemize}

\item a component-level admission model for heterogeneous responses under runtime support drift;

\item a runtime process combining dependency-consistent retention, exact dispatch, and fresh readmission;

\item layered evidence from structured-response replay, protocol-conformance tests, and PX4/Gazebo execution.

\end{itemize}

\begin{table*}[!t]
\centering
\tablefont
\setlength{\tabcolsep}{4pt}
\renewcommand{\arraystretch}{1.08}
\begin{tabular*}{\textwidth}{@{\extracolsep{\fill}}
P{0.21\textwidth}
P{0.18\textwidth}
P{0.25\textwidth}
P{0.28\textwidth}@{}}
\toprule
Research line & What it governs & Runtime role & Relation to HALO \\
\midrule

Guardrails / schemas
& Output or tool call
& Content and interface constraints
& Supply typed interfaces \\

AgentSpec
& Protected operation
& Programmable runtime decision
& HALO adds dependencies, late rechecks, and recovery \\

SAFEFLOW / Mnemosyne
& Workflow or transaction
& Coordination, commitment, and repair
& Govern broader workflow or transaction state \\

Freshness / assurance
& Evidence or controller action
& Freshness and downstream checks
& Supply component-specific support inputs \\

\textbf{HALO}
& \textbf{Component within one response}
& \textbf{Partial retention, final dispatch gate, scoped recovery}
& \textbf{Defines the proposed component runtime process} \\

\bottomrule
\end{tabular*}
\caption{Related approaches by what they govern and where they act in the runtime process.}
\label{tab:related-map}
\end{table*}

\section{Related Work}

Table~\ref{tab:related-map} compares prior approaches by what they govern and
where they act. Existing work controls generated content, individual
operations, larger workflows, or downstream safety signals. HALO instead
focuses on several dependent components returned together in one response.

Guardrails, agent-security benchmarks, sandboxed evaluation, and tool schemas
constrain generated content or external interfaces
\cite{toolemu2024,injecagent2024,llamaguard2023,nemoguardrails2023,
openai_function_calling,mcp_tools}. They help prevent invalid outputs from
reaching external systems, but do not decide which parts of one response should
remain when changing conditions affect only some components.

AgentSpec provides programmable checks for protected operations
\cite{agentspec2026} and is the closest executable runtime-policy comparator.
With matched inputs and current state, AgentSpec-Scoped reproduces HALO's
admission-time decision. HALO additionally preserves declared dependencies,
binds authorization to one exact action, rechecks it at the controlled
interface, and limits how a blocked component may return.

SAFEFLOW coordinates multi-agent information flow and transactions
\cite{safeflow2025}, while Mnemosyne studies admission, commitment, and repair
for generated workflow actions \cite{mnemosyne2026}. These systems govern
broader workflow or transaction state. HALO instead governs the components
within one response and preserves the still-valid subset without leaving
unsupported dependents.

Freshness metrics, robotic affordance checks, runtime assurance, shielding, and
control verification provide support signals or downstream checks
\cite{yates2021aoi,saycan2023,codeaspolicies2023,soter2019,
seto1998simplex,alshiekh2018shielding,ames2019cbf,
parasuraman2000model}. HALO uses these results as current support for individual
components rather than replacing them.

Selective agentic recovery for UAV autonomy decides when to invoke an external
reasoner and validates one returned recovery choice \cite{park2026pmr}. HALO
instead governs which components remain, which dependents are removed, which
exact actions may be dispatched, and how blocked components may re-enter.
Because these systems govern different objects and stages, we use AgentSpec and mechanism-removal variants to isolate HALO's runtime process.

\section{Method}

Figure~\ref{fig:halo-overview} turns the problem introduced in the
Introduction into HALO's runtime process. The figure follows one response from
its components to admission, dispatch, and recovery. HALO preserves outputs
that remain supported and removes outputs whose prerequisites are missing. It then rechecks each action before it crosses the controlled interface. A recoverable action
may return only as a fresh candidate.

\begin{figure*}[!t]
\centering
\includegraphics[width=\textwidth]{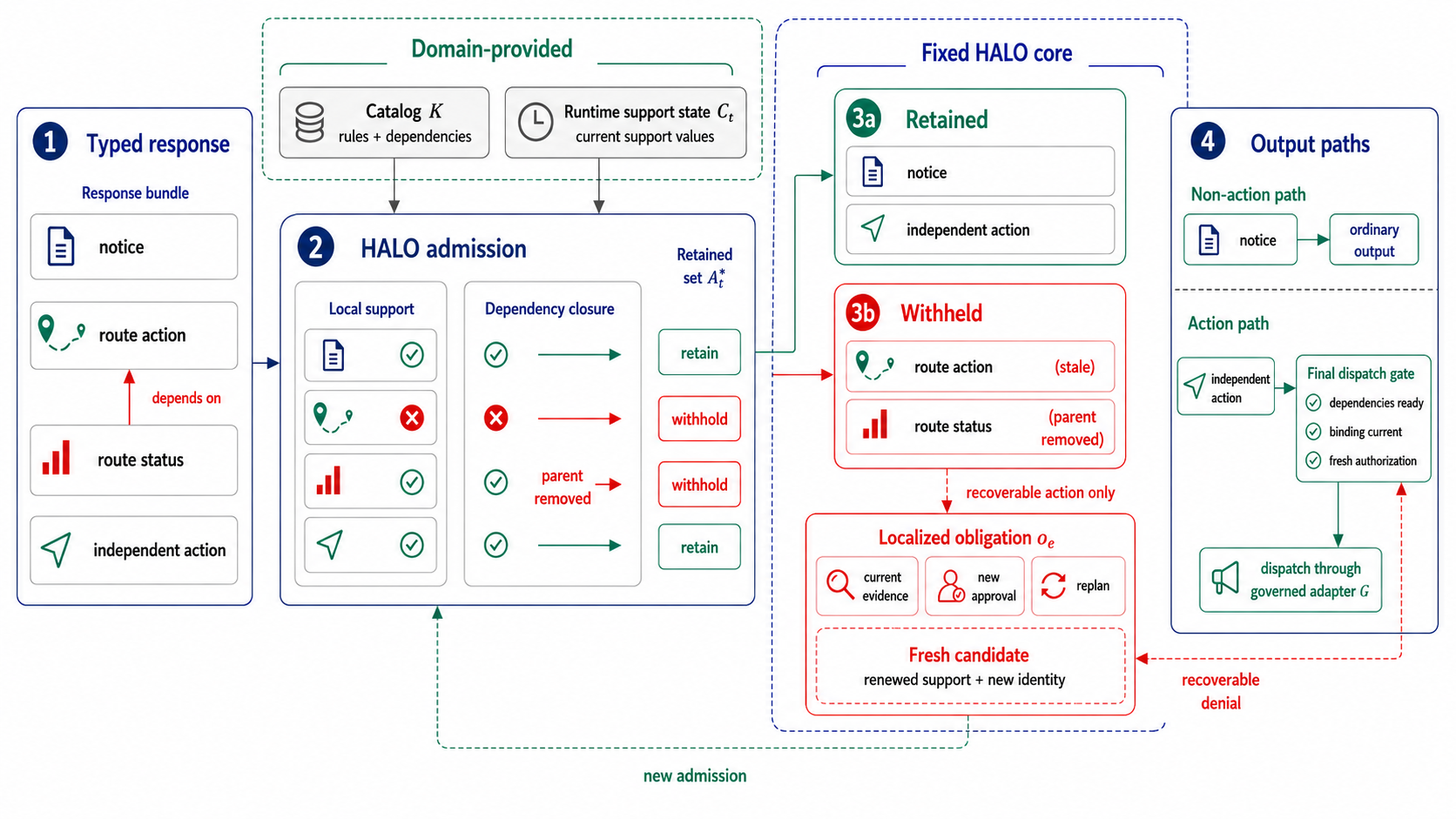}
\caption{HALO's runtime process for one structured response. HALO preserves
supported components under declared prerequisites, sends eligible actions to
the final dispatch gate, and allows blocked actions to be replaced only by fresh
candidates.}
\label{fig:halo-overview}
\end{figure*}

\subsection{Runtime Overview}

The agent returns one structured response. Each output has a declared type,
payload, identifier, and optional prerequisite references. We call each output
a \emph{component}, and write the complete response as $\widehat R$.

Figure~\ref{fig:halo-overview} also shows the running UAV example used
throughout this section. The response contains an operator notice, an
independent sensor report, a flight action, and a status message released only
after that action is dispatched. We denote them by $e_1$, $e_2$, $e_3$, and
$e_4$, respectively. Thus, $e_4$ requires $e_3$, while $e_1$ and $e_2$ remain
independent.

HALO does not infer types or prerequisites from free text. The agent proposes
them in the structured response, and HALO checks them against trusted rules.
A response object is \emph{well-formed} only when its type, identifier, payload,
and prerequisite references follow those rules. Malformed references, cycles,
and unsupported component types are rejected rather than guessed or silently
repaired.

\paragraph{Trusted rules.}
A trusted \emph{domain catalog} $K$ is a rule set defined outside the agent. It
defines the allowed component types, required prerequisites, support checks,
recovery routes, and output paths for one external interface. Trusted
\emph{state providers} supply the current evidence, authorization, versions,
provenance, approvals, and operating conditions needed by those checks. We
write this current snapshot as $C_t$.

The \emph{adapter} $G$ is the controlled software interface that sends an
authorized action to the external system. In the UAV setting, the adapter
publishes an authorized PX4 setpoint. A component that does not request an
external state change is a \emph{non-action}. HALO \emph{emits} it through its
configured output path. HALO \emph{dispatches} an action only through the
adapter. HALO does not replace sensing, planning, geometric verification, or
downstream vehicle control.

\paragraph{Component-specific support.}
Not every change matters to every component. A map update may affect the
flight action $e_3$, but not the operator notice $e_1$. We call the fields of
$C_t$ that matter to component $e$ its \emph{support footprint}, written
$\pi_e$. This lets HALO react only to changes that are relevant to each
component. The component-specific state is
\begin{equation}
\label{eq:footprint}
C_t^e=\pi_e(C_t).
\end{equation}
The catalog then checks whether the component is supported under the current
state:
\begin{equation}
\label{eq:local-support}
L_{C_t}(e)=\operatorname{LocalOK}_K(e,C_t^e).
\end{equation}
A change outside $\pi_e$ does not affect $e$. A change inside $\pi_e$ causes
HALO to reconsider that component and any components that require it.

\paragraph{Trust model.}
HALO trusts catalog $K$, faithful state providers, adapter $G$, and a trusted
\emph{authorization ledger}. The ledger records whether the authorization for
a dispatch attempt is fresh or has already been used. Returned payloads and
prerequisite references are untrusted. These assumptions define the boundary
within which HALO's guarantees apply. HALO's claims cover only the declared
response graph, trusted current-state inputs, and actions that cross the
enforced adapter.

Table~\ref{tab:halo-objects} provides a map of the notation used below.
Admission uses $\pi_e$ and $A_t^\star$, dispatch uses $b_e$ and $G$, and
recovery uses $o_e$.

\begin{table*}[!t]
\centering
\tablefont
\setlength{\tabcolsep}{5pt}
\renewcommand{\arraystretch}{1.08}
\begin{tabular*}{\textwidth}{@{\extracolsep{\fill}}
P{0.20\textwidth}P{0.18\textwidth}P{0.57\textwidth}@{}}
\toprule
Layer / object & Provided by & Role \\
\midrule

Response: $\widehat R$
& Agent
& Contains typed components and prerequisite references. \\

Response graph: $G_R$
& HALO
& Stores validated prerequisite relations and phases. \\

Catalog: $K$
& Domain designer
& Defines schemas, checks, paths, and recovery rules. \\

Current state: $C_t$
& State providers
& Supplies current evidence, authorization, and conditions. \\

Support footprint: $\pi_e$
& Catalog $K$
& Selects state fields relevant to component $e$. \\

Retained set: $A_t^\star$
& HALO
& Largest supported set closed under declared prerequisites. \\

Action binding: $b_e$
& HALO
& Binds one action instance to one authorized dispatch. \\

Adapter: $G$
& Domain integration
& Sends an action after final dispatch authorization. \\

Obligation: $o_e$
& HALO
& Records required updates and the allowed replacement. \\

\bottomrule
\end{tabular*}
\caption{Main objects in HALO's runtime process.}
\label{tab:halo-objects}
\end{table*}

\subsection{Component-Level Admission}

Table~\ref{tab:halo-objects} maps admission to the current state.
The support footprint $\pi_e$ selects the fields that matter to component $e$,
while $A_t^\star$ denotes the set retained by HALO.

Admission answers the first runtime question: \emph{which components may
remain together under the current state?} A component that passes admission is
\emph{retained}. Retention means that the component may remain in the runtime;
it does not yet authorize an action to be dispatched. A component that fails
admission is \emph{withheld}, meaning that HALO prevents it from reaching its
configured output or action path.

The agent proposes prerequisite references, and HALO validates their types,
directions, and required processing phases against catalog $K$. A validated
reference $(p,e,q)$ means that component $e$ may proceed only after prerequisite
component $p$ reaches processing phase $q$, such as emission or dispatch. Let
$D_\phi$ contain these validated references. The resulting response graph is
\begin{equation}
\label{eq:response-graph}
G_R=(V,D_\phi),
\end{equation}
where $V$ is the finite set of response components. In the running example,
$e_4$ requires $e_3$ to be dispatched.

Later checks must refer to the same admitted payload. HALO
canonicalizes each component and computes a fixed identifier,
called a \emph{digest}:
\begin{equation}
\label{eq:component-digest}
d_e=\operatorname{Digest}(\operatorname{Canonicalize}(e)).
\end{equation}
This prevents a modified payload from reusing earlier support or
authorization.

Local support alone is not enough. A component cannot remain when a declared
prerequisite has been removed. We call the repeated removal of such components
\emph{dependency closure}. Let
\[
\operatorname{Parents}(e)=\{p:\exists q,\,(p,e,q)\in D_\phi\}
\]
denote the declared prerequisite components of $e$. HALO retains the largest
set in which every component is currently supported and every declared
prerequisite also remains:
\begin{equation}
\label{eq:admission-fixed-point}
A_t^\star
=
\nu X.\,
\{e\in V:
L_{C_t}(e)\land
\operatorname{Parents}(e)\subseteq X\}.
\end{equation}
Here $\nu$ denotes the greatest fixed point. Operationally, HALO repeatedly
removes unsupported components and components whose prerequisites were removed
until no further change occurs. In the running example, stale flight action
$e_3$ and dependent status $e_4$ are removed, while notice $e_1$ and sensor
report $e_2$ remain.

\begin{lemma}[Correctness of dependency-closed retention]
\label{lem:closure}
For fixed catalog $K$ and current state $C_t$, after malformed and cyclic
references are rejected, $A_t^\star$ is the unique greatest locally supported
subset closed under all declared prerequisites.
\end{lemma}
This repeated-removal procedure reaches $A_t^\star$ in
$O(|V|+|D_\phi|)$ time after local support checks are evaluated. Proofs are in
the supplementary material.

\subsection{Dispatch Process}

Table~\ref{tab:halo-objects} separates retention from external use.
Binding $b_e$ identifies one action instance, while adapter $G$ sends it only
after the final dispatch gate succeeds.

\textbf{Admission retains; the final dispatch gate authorizes dispatch.}
A retained non-action, such as notice $e_1$, may be emitted after its
prerequisites are ready. A retained action remains only a candidate. Before it
is sent to the external system, HALO performs one last check. We call this
check the \emph{final dispatch gate}. It confirms that the same action instance
is still supported and authorized.

\paragraph{Processing-phase readiness.}
A component's \emph{processing phase} records its current progress, such as
retained, pending, emitted or dispatched, invalidated, or failed. Let
$\sigma_t(p)$ denote the phase of prerequisite component $p$, and let
$\succeq$ mean ``at or beyond.'' Component $e$ has ready prerequisites when
\begin{equation}
\label{eq:dependency-readiness}
\begin{aligned}
\operatorname{DepsReady}_t(e)
\iff{}&
\forall (p,q)\ \text{such that}\\
& (p,e,q)\in D_\phi,\quad
\sigma_t(p)\succeq q.
\end{aligned}
\end{equation}
A component remains pending until all required prerequisites reach their
specified phases. If a prerequisite fails permanently, HALO invalidates its
unemitted or undispatched dependents but leaves unrelated components unchanged.

\paragraph{Exact action binding.}
Conditions may change while an action waits, and a payload may be modified or
replayed. HALO must therefore recheck the same action instance that passed
admission. It records an \emph{admission witness} $W_e$, which identifies the
canonical payload and stable support data checked at admission, and an
\emph{action binding} $b_e$, which ties that witness to one dispatch attempt.
At admission time $t_a$, HALO creates
\begin{equation}
\label{eq:witness-binding}
\begin{aligned}
W_e &= \operatorname{CanonWitness}_K(d_e,C_{t_a}^{e}),\\
b_e &= H(W_e,\operatorname{seq}_e).
\end{aligned}
\end{equation}
$\operatorname{CanonWitness}_K$ selects stable fields, including the catalog
version, checked identities, and required certificate metadata. The associated
record also stores the admission generation and a \emph{one-dispatch token},
which authorizes at most one dispatch of that action. Values requiring a live
check remain in $C_t$ for the final dispatch gate. Together, these records prevent authority from being transferred
to a modified or replayed action.

\paragraph{Final dispatch gate.}
The final dispatch gate asks whether the exact retained action may cross the controlled
adapter boundary now. Immediately before dispatch, HALO evaluates
\begin{equation}
\label{eq:dispatch-ok}
\begin{aligned}
\operatorname{DispatchOK}_t(e)\iff{}&
e\in A_t^\star
\land \operatorname{DepsReady}_t(e)\\
&\land \operatorname{WitnessOK}_t(W_e)\\
&\land \operatorname{FreshAuth}_t(e,b_e).
\end{aligned}
\end{equation}
These conditions require current retention, ready prerequisites, unchanged
bound identities, and unused authorization for the same action binding.
Changes outside the support footprint remain irrelevant, while changes to
bound support data require readmission.

The final recheck, token consumption, and adapter call occur in one gate-owned
critical section. Its decision point $t_\ell$ is the instant when the gate
commits its result. If $\operatorname{DispatchOK}_{t_\ell}(e)$ holds, HALO
consumes the token and calls adapter $G$. A recoverable denial creates an
obligation; a terminal denial ends that action instance. This guarantee covers
the adapter call, not every downstream physical outcome.

\subsection{Localized Recovery}

Blocking an action is not sufficient; the runtime must also control how it may
return. As summarized in Table~\ref{tab:halo-objects}, HALO represents a
recoverable denial with a small recovery record called a \emph{localized
obligation} $o_e$. It records why the action was blocked, what support must be
refreshed, the permitted recovery route, and the \emph{replacement scope}
$S(o_e)$. The replacement scope identifies the blocked component and any
permitted fresh replacement.

An obligation performs no repair and grants no dispatch authority. Components
outside $S(o_e)$ keep their state. In the running example, recovery may replace
flight action $e_3$, but cannot authorize notice $e_1$ or sensor report $e_2$.

A catalog-defined \emph{recovery handler} may refresh support or produce a
revised action using updated evidence, a planner, or rule-based logic, but
cannot authorize dispatch. Any returned candidate must re-enter admission,
obtain a fresh binding and authorization, and pass the final dispatch gate.

\paragraph{Fresh readmission.}
A blocked action is never revived with its previous authority. Instead, a new
candidate undergoes \emph{fresh readmission} with current evidence, a later
generation, a new one-dispatch token, and an unused sequence.
Discharging $o_e$ only makes the candidate available; it must repeat
validation, support checking, dependency closure, exact binding, and the final
dispatch gate. Only components inside $S(o_e)$ may receive fresh instances;
ledger details and retry policies are in the supplementary material.

Algorithm~\ref{alg:halo-protocol} summarizes admission, processing-phase
readiness, final dispatch, and scoped fresh readmission. Relevant events
recompute only the affected dependency closure.

\begin{algorithm}[!t]
\caption{The HALO runtime protocol}
\label{alg:halo-protocol}
\tablefont
\begin{algorithmic}[1]
\REQUIRE Response $\widehat R$; catalog $K$; state providers; adapter $G$
\ENSURE Updated component states, emitted non-actions, dispatched actions,
and scoped obligations

\STATE \textbf{Admission:} Read current $C_t$; validate $\widehat R$ and its
prerequisite references; form $G_R$; compute $A_t^\star$; retain selected
components pending; create action bindings and one-dispatch tokens.

\WHILE{a support, expiry, path, or phase event affects $e$}
  \STATE Refresh $C_t$ and recompute the affected dependency closure.
  \IF{a required prerequisite of $e$ failed terminally}
    \STATE Invalidate $e$ and its unemitted or undispatched dependents.
  \ELSIF{$e\notin A_t^\star$}
    \STATE Withhold $e$ and dependents removed by closure.
    \STATE Create $o_e$ if the local denial is recoverable.
  \ELSIF{$\operatorname{DepsReady}_t(e)=0$}
    \STATE Keep $e$ pending.
  \ELSIF{$e$ is a non-action}
    \STATE Emit $e$ through its catalog-defined path; update $\sigma_t(e)$.
  \ELSE
    \STATE \textbf{Final dispatch gate:} Within the gate-owned critical section,
    refresh $C_t$, recompute closure, and evaluate
    $\operatorname{DispatchOK}_t(e)$.
    \IF{$\operatorname{DispatchOK}_t(e)=1$}
      \STATE Consume the token; invoke
      $G.\operatorname{dispatch}(e,b_e)$; update $\sigma_t(e)$.
    \ELSE
      \STATE Withhold $e$; create $o_e$ only if recoverable.
    \ENDIF
  \ENDIF
\ENDWHILE

\STATE \textbf{Recovery:} If $o_e$ remains live, readmit a candidate within
$S(o_e)$ with current support, a new one-dispatch token, a later generation, and an
unused sequence; otherwise terminate the obligation.
\end{algorithmic}
\end{algorithm}

\subsection{Protocol Properties}

These properties assume a fixed trusted catalog, faithful state providers, an
acyclic response graph, and enforced adapter and replacement-scope boundaries;
proofs are in the supplementary material.

\begin{corollary}[Preservation under irrelevant drift]
\label{cor:irrelevant-drift}
If local support remains unchanged for $e$ and every declared ancestor from
$t$ to $t'$, then
\begin{equation}
\label{eq:irrelevant-drift-preservation}
e\in A_t^\star
\;\Longrightarrow\;
e\in A_{t'}^\star.
\end{equation}
\end{corollary}
Thus, unrelated route drift preserves $e_1$ and $e_2$ rather than causing
whole-response rejection.

\begin{proposition}[Gate-boundary integrity]
\label{prop:gate-integrity}
If adapter $G$ is called for action $e$ at decision point $t_\ell$, then
\begin{equation}
\label{eq:gate-boundary-integrity}
\begin{aligned}
&e\in A_{t_\ell}^\star,\qquad
\operatorname{DepsReady}_{t_\ell}(e),\\
&\operatorname{WitnessOK}_{t_\ell}(W_e),\qquad
\operatorname{FreshAuth}_{t_\ell}(e,b_e)
\end{aligned}
\end{equation}
hold simultaneously, and the token bound to $b_e$ authorizes at most one
adapter call.
\end{proposition}
Thus, the final dispatch gate blocks stale or substituted action instances.

\begin{proposition}[Fresh scoped recovery]
\label{prop:scoped-recovery}
If an action derived from a blocked action reaches dispatch after discharge of
$o_e$, it has passed ordinary readmission with current support, a later
generation, a new one-dispatch token, and an unused sequence. Discharge creates
no recovered authority for any component outside $S(o_e)$.
\end{proposition}

\subsection{UAV Instantiation}
\label{sec:domain-instantiations}

A domain supplies the catalog, state providers, recovery handler, and adapter
for one external interface. These details change, but HALO's core runtime
process remains the same.

The UAV catalog uses time-stamped evidence, operational rules, map or traffic
epochs, geometric bounds, approval, authorization, and vehicle state. For
source $s$, age of information (AoI) is the time since its latest valid update:
\begin{equation}
\label{eq:uav-aoi}
\Delta_s(t)=t-u_s(t),
\end{equation}
where $u_s(t)$ is that update time \cite{yates2021aoi}. Catalog $K$ assigns the
source-specific limits. POINT checks support at one instant, INTERVAL requires
support throughout a declared interval, and ROLLING rechecks support before
each new sequence. Together with the operational rules, these checks
instantiate $\operatorname{LocalOK}_K$ in Eq.~\eqref{eq:local-support}.

Selected Federal Aviation Administration (FAA) Part~107 provisions inform
catalog rules for pilot authority, airspace, operating limits, preflight
support, and traffic separation \cite{ecfr_part107}. The 410 PX4 profile
replays test agreement with these rules, not regulatory compliance. The final
dispatch gate rechecks epochs, bounds provenance, corridor support, exact
identity, and authorization. The adapter call is the formal boundary; PX4
setpoints are downstream observations, not guarantees of motion.

\begin{table*}[!t]
\centering
\tablefont
\setlength{\tabcolsep}{4pt}

\textbf{(a) Structured-response replay: 248 supported components; 72 epoch,
72 bounds, and 70 prerequisite opportunities}\par
\begin{tabular*}{\textwidth}{@{\extracolsep{\fill}}lrrrr@{}}
\toprule
Variant & Supported kept$\uparrow$ & Stale-epoch leak$\downarrow$ &
Wrong-bounds leak$\downarrow$ & Orphan leak$\downarrow$ \\
\midrule
AgentSpec-Global
& 0/248 & 0/72 & 0/72 & 0/70 \\
AgentSpec-Scoped\textsuperscript{\dag}
& 248/248 & 0/72 & 0/72 & 0/70 \\
IndependentFilter
& 248/248 & 0/72 & 0/72 & 70/70 \\
HALO-Admission
& 248/248 & 72/72 & 72/72 & 0/70 \\
\textbf{HALO-Full}
& \textbf{248/248} & \textbf{0/72} & \textbf{0/72} & \textbf{0/70} \\
\bottomrule
\end{tabular*}

\smallskip
\textbf{(b) PX4/Gazebo: 240 supported components; 50 stale, 100 prerequisite,
and 200 overblocking opportunities}\par
\begin{tabular*}{\textwidth}{@{\extracolsep{\fill}}lrrrr@{}}
\toprule
Variant & Supported kept$\uparrow$ & Stale setpoint$\downarrow$ &
Orphan leak$\downarrow$ & Supported blocked$\downarrow$ \\
\midrule
WholeResponse
& 40/240 & 0/50 & 0/100 & 200/200 \\
IndependentFilter
& 240/240 & 0/50 & 100/100 & 0/200 \\
HALO-Admission
& 240/240 & 50/50 & 0/100 & 0/200 \\
\textbf{HALO-Full}
& \textbf{240/240} & \textbf{0/50} & \textbf{0/100} & \textbf{0/200} \\
\bottomrule
\end{tabular*}

\caption{Mechanism comparison. WholeResponse overblocks,
IndependentFilter leaves orphaned dependents, and HALO-Admission misses
post-admission drift. AgentSpec-Scoped matches the admission-time decision but
lacks HALO's later binding, gate, and recovery process. HALO-Full avoids all
tested failures. \textsuperscript{\dag}\,Matched admission-time control.}
\label{tab:core-ablations}
\end{table*}

\section{Evaluation}

RQ1 tests mechanism necessity, RQ2 protocol conformance, RQ3 fresh scoped
recovery, and RQ4 runtime cost.

\subsection{Setup}

Paired variants share the response, catalog, graph, mutation, timing, and
adapter. WholeResponse removes partial retention, IndependentFilter removes
dependency closure, and HALO-Admission removes the final recheck. HALO-Full
keeps the complete runtime process. AgentSpec-Global applies one response-level
decision, while AgentSpec-Scoped shares HALO's component-scoped predicates and
current state. Separate controls test modification and replay.

Experiments use Python~3.10 on Windows Subsystem for Linux~2 (WSL2), PX4
software-in-the-loop (SITL) \cite{meier2015px4}, Gazebo
\cite{koenig2004gazebo}, live OpenClaw server-sent events (SSE), and a
Crazyflie~2.1 platform for the physical UAV sequence.
Cold-start sessions use a new PX4 process; AgentSpec steady-state repetitions
do not. The replay, PX4-profile, controlled-timing, and AgentSpec panels are
separate cohorts whose denominators are not pooled. We report case-level
outcomes separately from PX4 setpoints.

\subsection{RQ1: Mechanism Necessity}

Table~\ref{tab:core-ablations} should be read across each reduced variant: its
failure column identifies the missing runtime mechanism. WholeResponse
overblocks, IndependentFilter leaves dependents without required parents, and
HALO-Admission misses post-admission drift; HALO-Full avoids all three in the
evaluated cases. Across 300 calls in 15 families, 226 replay-eligible responses
produced 7,910 paired evaluations, while separate controls rejected modified
instances and reused authorization. The PX4/Gazebo panel evaluates 12
predefined conditions under four protocol variants across ten cold-start SITL
sessions, yielding 480 condition--variant executions. RQ3 additionally presents
one controlled UAV operation. Together, these results show that retention,
dependency closure, exact binding, and final-gate revalidation address distinct
failures.

\paragraph{Comparison with AgentSpec.}
An action \emph{fingerprint} is a log identifier used to match an action with
an internal PX4 setpoint. With matched inputs and current state,
AgentSpec-Scoped reproduced HALO's admission-time decision. Without HALO's
binding and final dispatch gate, it produced fingerprints for stale actions in
all tested post-decision cases and accepted all replayed tokens. Adding HALO's
runtime process removed both failures and completed all fresh recovery cases.
A separate cold-start confirmation showed the same pattern; detailed
denominators are provided in the supplementary material. This isolates HALO's contribution as the
post-decision dispatch and recovery process rather than a richer rule language.

\subsection{RQ2: Runtime Protocol Conformance}

We ask whether the implementation matches the formal protocol. It matched all
expected outcomes across 96 admission cases, 135 exhaustively enumerated graphs
of up to 16 components, 45 independently checked graphs of 32 components, 20
protocol cases, and 10,000 valid or fault-injected schedules. These cases
exercise local support, closure, event ordering, final rechecks, one-dispatch
token enforcement, and recovery transitions, providing protocol-conformance
rather than open-world success evidence. Agreement across exhaustive small
graphs and independently checked larger graphs indicates that conformance is
not tied to one graph size or test construction.

\begin{figure*}[!t]
\centering
\includegraphics[width=\textwidth]
{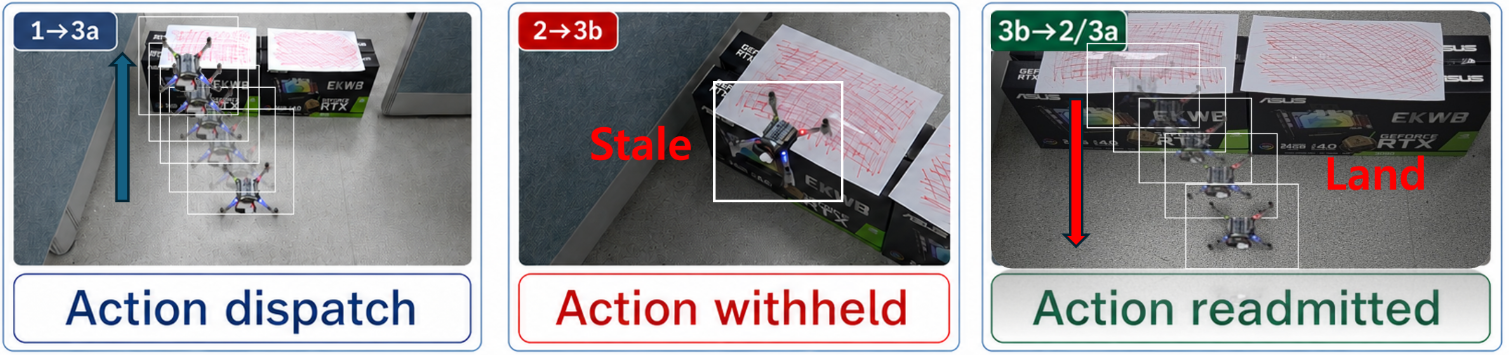}
\caption{Physical Crazyflie realization of Figure~\ref{fig:halo-overview}.
HALO first retains the forward action, then withholds it after support changes
while preserving the independent range report, and finally readmits a fresh
recovery candidate.}
\label{fig:uav-operation}
\end{figure*}

\subsection{RQ3: Fresh Scoped Recovery}

The PX4 profiles combine UAV-specific catalog conditions---source-specific
freshness, POINT/INTERVAL/ROLLING semantics, operational rules, authorization,
epochs, and geometric support---with HALO's exact-identity and unused-sequence
checks. Recovery requires current evidence, a later generation, a new
one-dispatch token, and an unused sequence; terminal or out-of-scope cases
create no authority. Across 48/48 recovery-oracle matches, all 36 executable
cases returned through admission, while 12 terminal cases ended without
authority. Controls rejected five attempts to reuse prior authorization and
one out-of-scope recovery.

\paragraph{Physical operation.}
Figure~\ref{fig:uav-operation} instantiates the three transitions in
Figure~\ref{fig:halo-overview}: the initial forward action is retained after
admission ($1\!\rightarrow\!3a$), the unsupported forward action is withheld
($2\!\rightarrow\!3b$) while \texttt{REPORT(z\_range)} continues to be
emitted, and a fresh backtrack-and-land candidate re-enters admission,
receives a new binding and authorization, and passes the final dispatch gate
($3b\!\rightarrow\!2/3a$). Table~\ref{tab:domain-summary} quantifies this Crazyflie
sequence and summarizes the broader PX4 and recovery checks.

\begin{table}[!t]
\centering
\footnotesize
\setlength{\tabcolsep}{4pt}
\renewcommand{\arraystretch}{1.10}
\begin{tabular}{@{}
P{0.30\columnwidth}
P{0.64\columnwidth}@{}}
\toprule
Test & Observed result \\
\midrule

Crazyflie operation
& 10/10 stale actions withheld; 10/10 range reports emitted;
0/10 matching stale setpoints; 10/10 fresh readmissions. \\

PX4 dispatch checks
& 410/410 decisions matched; 0/250 denied actions reached PX4;
60/60 recoveries were fresh. \\

Recovery checks
& 48/48 outcomes matched; 12 terminal; 36 returned through admission. \\

\bottomrule
\end{tabular}
\caption{Repeated Crazyflie, PX4, and recovery outcomes;
the first row corresponds to Figure~\ref{fig:uav-operation}.}
\label{tab:domain-summary}
\end{table}

\subsection{RQ4: Runtime Cost}

We finally ask how the core checking cost grows with response size after
domain-specific predicates have been evaluated. Admission takes
\begin{equation}
\label{eq:admission-complexity}
O(|V|+|D_\phi|),
\end{equation}
and the final dispatch gate is linear in action $e$'s declared prerequisites
and footprint size $|\pi_e(C_t)|$. With up to 32 components, the largest median
admission cost was 0.333\,ms and the 95th percentile was 0.453\,ms. With up to
16 direct dependencies and 32 footprint fields, the largest median
pre-consumption gate cost was 0.035\,ms and the 95th percentile was
0.061\,ms. Median-based regressions yielded $R^2=0.9978$ for admission and
$R^2=0.9975$ for final-gate compute, consistent with but not proving the
source-level bounds. Measurements exclude model generation, domain predicates,
token mutation, and downstream execution; the tested core checks are therefore
small in this setting without establishing general scalability.

\section{Discussion}

HALO separates two decisions often conflated: whether a component remains
eligible and whether an exact action may cross the controlled interface.
Partial preservation requires dependency-consistent retention, while admission
alone cannot protect a waiting action as runtime support changes. Localized
obligations identify what must be refreshed without reviving the blocked
instance or expanding its authority.

HALO's guarantees are boundary-scoped. Under a trusted catalog, faithful state
providers, an enforced authorization ledger, and a gate-owned adapter, HALO provides dependency coherence, exact-instance binding, one-dispatch token
enforcement, and fresh scoped readmission. Semantic correctness, undeclared
dependencies, provider failures, distributed atomicity, and downstream
physical outcomes remain outside this boundary. PX4 setpoints record governed
interface crossings; the Crazyflie motion remains a downstream observation.

The contribution is not a new predicate language, freshness metric, or
authorization primitive, but a reusable typed-component runtime protocol above existing runtime-policy and domain-assurance mechanisms. The
same process can apply when a domain supplies support rules, current state
providers, and an enforced effect boundary. A supplementary cross-domain check
shows that HALO can be applied at a versioned software boundary without losing
the tested preservation, blocking, or fresh recovery behavior.

\section{Conclusion}

Components from one agent response do not share one validity state. HALO
preserves the largest supported set whose declared prerequisites also remain
supported, then separately determines whether each exact action has current
dispatch authority. A localized obligation records required updates and replacement scope, but
only a fresh candidate may return.

Across structured replay, protocol-conformance tests, and PX4/Gazebo
experiments, HALO-Full preserved unaffected components, removed unsupported
dependents, excluded all tested stale dispatches, and completed fresh scoped
recovery. These results establish heterogeneous component admission as a
distinct runtime layer between structured responses and controlled external
actions. As LLMs evolve into agents that act on external systems, HALO provides a
component-level runtime paradigm for governing effect-bearing responses.

\appendix

\maketitle
\suppressfloats[t]

\section{Scope of This Supplement}
This supplement expands the formal arguments, evaluation cohorts, measurement
boundaries, and reproduction conditions behind the four research questions in
the main paper: mechanism necessity, runtime-protocol conformance, fresh scoped
recovery, and runtime cost. It uses the main paper's notation:
$\widehat R$ is the agent-returned bundle, $G_R=(V,D_\phi)$ is the
catalog-validated response graph, $K$ is the trusted domain catalog, and
$C_t$ is the time-indexed runtime support state. The primary evidence remains
the structured-response and UAV/PX4 evaluation reported in the main paper.
A final, explicitly supplementary software-boundary check evaluates whether
the same runtime process can be instantiated with a different catalog and
adapter; it is not an additional main-paper research question.

\section{Formal Model and Protocol}

\subsection{Core and Domain Interface}
HALO defines a fixed component-level protocol over a domain interface. Its core
governs dependency closure, processing state, exact dispatch authority, and
scoped recovery; domains supply component types,
footprints, predicates, recovery routes, providers, and governed adapters.
Table~\ref{tab:domain-interface} first records the main-paper UAV
instantiation and then the supplementary repository instantiation. The latter
tests interface transfer at a sandboxed versioned-software boundary; it does
not claim a new merge algorithm or replace Git's native consistency
mechanisms.

\begin{table*}[t]
\centering
\footnotesize
\setlength{\tabcolsep}{3pt}
\renewcommand{\arraystretch}{1.04}
\begin{tabular*}{\textwidth}{
@{\extracolsep{\fill}}
P{0.15\textwidth}
P{0.25\textwidth}
P{0.27\textwidth}
P{0.27\textwidth}@{}}
\toprule
Interface element & Fixed HALO role & UAV instantiation
& Supplementary repository instantiation \\
\midrule
Component types
& Units of admission and processing state
& notice, refresh request, route action, status
& review notice, evidence request, merge action, completion notice \\
Support footprint
& Runtime-state fields relevant to one component
& traffic age, map epoch, vehicle state, bounds
& target head, patch digest, CI, approval \\
Support predicates
& Inputs to $\operatorname{LocalOK}$
& freshness, authority, geometry, operating rules
& version, protection, CI, approval \\
Dependency phases
& Parent milestones required before progress
& route dispatch or observed completion
& merge dispatch before completion notice \\
Recovery route
& Source of a fresh candidate
& evidence refresh, approval, replan
& updated CI, approval, rebase \\
Predicate provider
& Supplies the current runtime support state
& telemetry and runtime services
& repository and review state \\
Governed adapter
& Gate-owned effect boundary
& PX4 setpoint adapter call
& merge-commit creation call \\
Observation hook
& Downstream empirical evidence
& PX4 internal setpoint
& Git ancestry audit \\
\bottomrule
\end{tabular*}
\caption{The fixed HALO interface under the main-paper UAV instantiation and
the supplementary repository transfer check. The two instantiations change
the catalog, providers, recovery route, and governed adapter while preserving
the same HALO protocol core.}
\label{tab:domain-interface}
\end{table*}

\subsection{Component-Level Admission}
The trusted catalog, rather than the returned bundle, determines each type's
schema, canonicalization, mandatory footprint, support predicates, allowed
dependency phases, certificate requirements, and recovery route. The response
supplies only typed payloads and component references. This is the assumption
under which the following oracle comparison is interpreted.

\paragraph{Typed binding model.}
The agent returns the typed bundle $\widehat R$; HALO validates its component
types and dependency references against catalog $K$ and constructs the finite
phase-labeled graph
\begin{equation}
\label{eq:supp-response-graph}
G_R=(V,D_\phi),
\qquad
D_\phi\subseteq V\times V\times\mathcal{Q},
\end{equation}
where \(V=\{e_1,\ldots,e_n\}\) is the set of typed components and
\((p,e,q)\in D_\phi\) states that dependent \(e\) requires parent \(p\) to
reach phase \(q\in\mathcal{Q}\).  We write
\begin{equation}
\label{eq:supp-dependency-sets}
\begin{aligned}
D_\phi(e)&=\{(p,q):(p,e,q)\in D_\phi\},\\
D(e)&=\{p:\exists q,\,(p,q)\in D_\phi(e)\}.
\end{aligned}
\end{equation}
for the phase-labeled and positive prerequisite dependencies of \(e\),
respectively.
Duplicate identifiers, missing references, and cycles fail closed before a
dispatch authorization is issued.

For each component, the runtime first computes the canonical digest
\begin{equation}
\label{eq:supp-component-digest}
d_e=\operatorname{Digest}(\operatorname{Canonicalize}(e)).
\end{equation}
The catalog-defined footprint selects the current state relevant to $e$:
\begin{equation}
\label{eq:supp-footprint}
C_t^e=\pi_e(C_t).
\end{equation}
The stable subprojection
\(\beta_e(C_t)\) contains source identities, epochs or versions, bounds and
provenance identifiers, and approval or certificate identity. Admission stores
\begin{equation}
\label{eq:supp-stable-witness}
\begin{aligned}
W_e
&=\operatorname{CanonWitness}_K\bigl(
  d_e,\beta_e(C_{t_a}),\\
&\hspace{30mm}\operatorname{certificate}_e,v_K\bigr).
\end{aligned}
\end{equation}
The exact action binding used by the main paper is
\begin{equation}
\label{eq:supp-binding}
b_e=H(W_e,\operatorname{seq}_e).
\end{equation}
The associated admission record stores the component type, footprint,
dependency references, admission generation, and a \emph{one-dispatch token}
bound to this exact action instance. The token authorizes at most one dispatch. These fields support validation and ledger checks;
the dispatch identity itself is $b_e$.
Clock time, age, and changing telemetry are not hashed as immutable witness
bytes; catalog predicates reevaluate them from the live projection at the
final gate. HALO distinguishes
generation \(t_g\), admission \(t_a\), final authorization \(t_c\), first
enqueue \(t_p\), and optional downstream observation \(t_x\).  The binding at
dispatch checks the witness against current support rather than hashing the
proposal-wide mutable state.

\paragraph{Temporal support semantics.}
The following equations restate the main-paper UAV definitions so that the
temporal conventions are self-contained.
For source \(s\), let \(u_s(t)\) be the generation time of its newest available
update. Its Age of Information is \cite{yates2021aoi}
\begin{equation}
\label{eq:supp-aoi}
\Delta_s(t)=t-u_s(t).
\end{equation}
For component \(e\), let \(\mathcal{S}_e\) be its relevant sources and
\(H_{e,s}\) their catalog thresholds. The represented freshness predicate is
\begin{equation}
\label{eq:supp-fresh-ok}
\operatorname{FreshOK}(e,t)\iff
\bigwedge_{s\in\mathcal{S}_e}\Delta_s(t)\leq H_{e,s}.
\end{equation}
Equality is admitted: \(\Delta_s(t)=H_{e,s}\) is valid. AoI is distinct from
source epoch, bounds provenance, certificate validity, information
correctness, and operational-rule truth.

A separate \(\operatorname{RuleOK}\) predicate checks the represented
operational condition. The UAV adapter combines the two dimensions:
\begin{equation}
\label{eq:supp-support-ok}
\begin{aligned}
\operatorname{SupportOK}(e,t)
&\iff \operatorname{FreshOK}(e,t)\\
&\quad{}\land\operatorname{RuleOK}(e,t).
\end{aligned}
\end{equation}
Freshness therefore cannot substitute for rule truth, and rule truth cannot
substitute for current evidence.

The temporal certificate determines the time region that support must cover.
For a POINT component, freshness is checked at its first downstream effect:
\begin{equation}
\label{eq:supp-point}
\operatorname{POINT\_OK}(e)\iff
\operatorname{FreshOK}(e,t_{\mathrm{first}}).
\end{equation}
For an INTERVAL component, support is certified to cover the declared interval
by the conservative condition
\begin{equation}
\label{eq:supp-interval}
t_{\mathrm{end}}+\delta_g<t_{\mathrm{valid\_until}}(e),
\end{equation}
where \(\delta_g\) is the adapter guard. Equality fails closed. This is an
interval-coverage check over the certificate horizon, rather than continuous
live monitoring during execution. For a ROLLING component, every tested window
requires refreshed support and a new admission before the next action-bearing
sequence. The following profiles are informed by selected FAA Part~107
provisions \cite{ecfr_part107}.

\begin{table*}[t]
\centering
\small
\setlength{\tabcolsep}{4pt}
\begin{tabular*}{\textwidth}{@{\extracolsep{\fill}}
lP{0.11\textwidth}P{0.44\textwidth}P{0.28\textwidth}@{}}
\toprule
Profile & Anchor & Represented machine-checkable support & Fresh recovery \\
\midrule
AUTH-1 & 14 CFR 107.19
& Remote-PIC designation, current authority token, and override or abort
availability
& Fresh authority evidence and new admission \\
AUTH-2 & 14 CFR 107.41
& Airspace authorization not required, or represented authorization present
and current
& Fresh authorization evidence and new admission \\
AUTH-3 & 14 CFR 107.51
& Operation inside represented time, geometry, altitude, and configured
operating-limit scope
& Corrected scope and new admission \\
OPS-1 & 14 CFR 107.49
& Current environment assessment, healthy command link, sufficient energy, and
current preflight evidence
& Refreshed preflight support and new admission \\
SEP-1 & 14 CFR 107.37
& Current traffic epoch and age, available bounded tube, and represented
separation predicate
& Fresh traffic, a new tube, and new admission \\
\bottomrule
\end{tabular*}
\caption{FAA-anchored engineering profiles used by the benchmark. Each profile
maps a selected operational provision to machine-checkable support conditions
and a fresh recovery route.}
\label{tab:supp-faa-profiles}
\end{table*}

\paragraph{Dependency-closed admission.}

For fixed catalog $K$ and runtime support state $C_t$, let
\[
L_{C_t}(e)=\operatorname{LocalOK}_K(e,\pi_e(C_t))
\]
denote component-local support, excluding dependency checks, and let $D(e)$
denote declared parents. The corresponding closure operator is
\begin{equation}
\label{eq:supp-closure-operator}
F_{C_t}(X)=\{e\in V: L_{C_t}(e)\ \land\ D(e)\subseteq X\}.
\end{equation}
At dispatch, the declared relation is phase-labeled: a dependent component
must wait for the required parent phase, not merely for parent admission.
The final gate also checks the exact canonical component digest, current
component support projection, certificate, horizon, and unused one-dispatch token.
Under unique component identifiers, a finite acyclic declared graph, a fixed
runtime support state, independently valid unused component-bound tokens, and no
cross-component resource-allocation conflict, the greatest fixed point of
$F_{C_t}$ is the greatest dependency-closed locally supported subset. The
implementation rejects duplicate identifiers, missing references, and cycles
before an effect can be emitted. The broader graph-oracle benchmark below tests
this construction independently across multiple graph sizes.

\paragraph{Matched admission-time control.}
A customized policy using the same component predicates and declared
prerequisites can reproduce HALO's initial retained set. We therefore use this
configuration only as a matched admission-time control. It does not provide
HALO's exact-instance binding, final dispatch recheck, component-bound
one-dispatch token enforcement, dependency-aware processing, or scoped fresh
readmission.

The expanded dependency-DAG oracle benchmark contains 180 deterministic cases.
For 135 graphs with 4, 8, or 16 components, exhaustive subset enumeration is
the oracle. For 45 graphs with 32 components, a separately implemented
descending-rejection oracle is used. The oracle code does not import HALO's
admission procedure or share its selection helpers, uses a different
rejection-propagation traversal, and is driven by the frozen graph artifact
rather than HALO's traversal order. Its
maximality check starts from locally invalid, catalog-invalid,
missing-reference, and cyclic nodes, then rejects every descendant whose parent
has been rejected. The harness checks local soundness, dependency closure, and
greatest-set maximality separately from exact-set agreement. The graph-generation
seed is recorded independently; the oracle execution is deterministic and
performs no randomized sampling. The tested grid has zero local-validity,
dependency-closure, and greatest-set mismatches. Once local predicate results are available,
dependency closure requires \(O(|V|+|D_\phi|)\) time. Including predicate
evaluation, total admission cost is
\(O(|V|+|D_\phi|+\operatorname{PredicateCost})\).

\subsection{Dispatch and Recovery Details}

\paragraph{Component processing state.}
HALO separates the current runtime support state $C_t$, the exact action
binding $b_e$, and processing state $\sigma_t(e)$. State $C_t$ contains mutable
world and system support. The admission record contains the canonical digest,
component footprint, support witness, certificate metadata, catalog version,
token, sequence, and admission generation; $b_e$ identifies the exact action
instance and dispatch attempt. The gate instance fixes the governed adapter.
The processing state records whether that exact
component instance has progressed far enough to satisfy a phase-labeled dependency.

The represented component states are
\begin{equation}
\label{eq:supp-component-states}
\begin{aligned}
\sigma_t(e)\in\{&
\mathsf{PROPOSED},\mathsf{RETAINED},\mathsf{PENDING},\\
&\mathsf{EMITTED},\mathsf{DISPATCHED},\\
&\mathsf{OBSERVED\_SUCCESS},\mathsf{WITHHELD},\\
&\mathsf{FAILED},\mathsf{INVALIDATED}\}.
\end{aligned}
\end{equation}
$\mathsf{WITHHELD}$ denotes a component prevented from reaching its
catalog-defined output path. $\mathsf{INVALIDATED}$ denotes an undispatched
retained instance that cannot
proceed under its current binding. $\mathsf{FAILED}$ denotes an instance that passed
the final gate but failed to reach a required downstream milestone. A dependent
is ready only when every parent reaches its catalog-required phase:
\begin{equation}
\label{eq:supp-dependency-readiness}
\begin{aligned}
\operatorname{DepsReady}(e,\sigma_t)
&\iff \forall(p,q)\in D_\phi(e):\\
&\hspace{22mm}\sigma_t(p)\succeq q.
\end{aligned}
\end{equation}
A terminal parent failure invalidates unemitted or undispatched descendants. It does not roll
back already emitted or dispatched components or unrelated lineages.
A declared parent also remains a current support prerequisite until its
dependent proceeds; the phase label adds the processing milestone checked by
\(\operatorname{DepsReady}\).

Obligations use a separate recovery state
\begin{equation}
\label{eq:supp-obligation-states}
\begin{aligned}
\mathsf{CREATED}&\rightarrow\mathsf{DISPATCHED}
\\[-1mm]
&\rightarrow\mathsf{EXTERNAL\_SUPPORT\_RECEIVED}\\[-1mm]
&\rightarrow\mathsf{READY\_FOR\_READMISSION}\\[-1mm]
&\rightarrow\mathsf{DISCHARGED},
\end{aligned}
\end{equation}
with explicit failure and timeout exits. Readiness is evidence that the routed
support arrived, not dispatch authorization:
\begin{equation}
\label{eq:supp-readiness-not-dispatch}
\begin{aligned}
\omega_t(o_e)
&=\mathsf{READY\_FOR\_READMISSION}\\
&\not\Rightarrow\operatorname{Dispatch}(e).
\end{aligned}
\end{equation}

\paragraph{Executable admission procedure.}
The following numbered procedure is the supplement counterpart of the compact
method equations in the main paper.
\begin{enumerate}
\item Validate the typed envelope, catalog version, identifiers, references,
and dependency phases. Reject duplicate identifiers, unknown types, missing
references, cycles, and catalog mismatches.
\item Canonicalize each component and compute its exact digest. Check that any
support certificate is bound to that digest.
\item Read one runtime support state and evaluate the catalog predicates over each
component's required footprint.
\item Initialize the candidate set with locally supported components. Repeatedly
remove a component whose declared prerequisite parent is absent; the remaining
set is $A_t^\star$.
\item For each retained action-bearing component, create a binding over its
canonical bytes, component-relevant support witness, certificate, catalog
version, sequence, and admission generation. Issue a component-bound
one-dispatch token.
\item On a support or parent-phase event, reevaluate the affected component and
its dependent closure. An unmet nonterminal phase leaves the component pending;
a terminal failure or currently unsupported ancestor invalidates unemitted or undispatched
descendants.
\item At the sole governed-adapter entry, read current support, check membership in
the current dependency closure, validate the stable witness, dependency phases,
horizon, and unused current-generation token, then atomically consume the token
and invoke the adapter.
\item On a recoverable denial, create the catalog-routed obligation. External
support may create a fresh candidate, but that candidate returns to Step~1 with
a later generation, current evidence, a new one-dispatch token, and an unused
sequence. A
terminal denial produces no dispatch authorization.
\end{enumerate}

\paragraph{Proof details.}
\begin{theorem}[Greatest dependency-closed admission]
\label{thm:greatest-admission}
For a finite acyclic represented graph with fixed catalog $K$ and runtime
support state $C_t$,
$A_t^\star=\nu X.F_{C_t}(X)$ is the unique greatest locally supported set closed
under declared positive dependencies.
\end{theorem}
\begin{proof}
The operator is monotone on the finite lattice $2^V$ and maps every subset to
a subset of the locally supported set. Descending iteration from the locally
supported set therefore stabilizes at a fixed point. Any other locally
supported dependency-closed set is contained in the initial set and, by
monotonicity, in every subsequent iterate. It is therefore contained in the
limit $A_t^\star$, establishing greatestness and uniqueness.
\end{proof}

\begin{theorem}[Preservation under irrelevant drift]
\label{thm:irrelevant-drift}
Suppose a mutation from $C_t$ to $C_{t'}$ leaves the support projection and
local support of $e$ and every declared ancestor unchanged. If
$e\in A_t^\star$, then $e\in A_{t'}^\star$.
\end{theorem}
\begin{proof}
The mutation preserves each local predicate result in the finite ancestor
closure. Induction from roots to $e$ preserves membership in the corresponding
fixed point, so $e$ remains retained.
\end{proof}

\begin{theorem}[Gate-boundary integrity]
\label{thm:gate-integrity}
Assume the final gate is the sole governed adapter entry and predicate
providers faithfully represent the state checked at the gate decision point.
If the component is outside the current dependency closure, a required phase
is unmet, its stable witness is invalid, or its dispatch authorization is
stale at $t_\ell$, no adapter call is made for that component instance.
Conversely, if adapter $G$ is called for $e$ at $t_\ell$, then
\begin{equation}
\label{eq:supp-gate-boundary-integrity}
\begin{aligned}
e\in A_{t_\ell}^\star
&\land\operatorname{DepsReady}_{t_\ell}(e)\\
&\land\operatorname{WitnessOK}_{t_\ell}(W_e)\\
&\land\operatorname{FreshAuth}_{t_\ell}(e,b_e)
\end{aligned}
\end{equation}
holds, and the token bound to $b_e$ authorizes at most one call.
\end{theorem}
\begin{proof}
Every checked mismatch falsifies a conjunct of
$\operatorname{DispatchOK}_{t_\ell}(e)$. The gate therefore denies before
token consumption and the adapter call. When every conjunct holds, the gate
commits the decision, consumes the component-bound token, and invokes $G$
within the gate-owned critical section. Because the gate is the adapter's sole
entry and the token authorizes one dispatch, a denied instance has no alternate
governed path and an accepted binding authorizes at most one call.
\end{proof}

\begin{theorem}[No revival of prior authorization]
\label{thm:no-prior-revival}
Let $b_b$ be a withheld, undispatched binding. If recovery passes through the
obligation process and governed dispatch uses the same final gate, then
$b_b$ cannot authorize a later dispatch.
\end{theorem}
\begin{proof}
Obligation transitions issue no token. Fresh admission increments the
generation, rejects every prior sequence, and revokes prior-generation tokens.
It may retain certificate bytes only when the certificate remains current and
bound to the exact component. The final gate accepts only the current component
instance and its unused current-generation token, so any recovered dispatch
must use $b_r\neq b_b$.
\end{proof}

\begin{theorem}[Recovery locality]
\label{thm:recovery-locality}
Let \(S(o_e)\) be the replacement scope carried by obligation \(o_e\). Discharging
\(o_e\) creates no new binding, token, admission generation, or recovered
component instance for any \(u\notin S(o_e)\). A component outside the scope
may advance only under its existing binding and the ordinary dependency
rules.
\end{theorem}
\begin{proof}
The obligation manager accepts fresh candidates only when their component
identifiers are contained in \(S(o_e)\), and obligation transitions themselves
issue no dispatch authorization. The fresh-admission transition updates
only the accepted candidate identifiers. A trusted dispatcher passes that
validated candidate set unchanged to fresh admission. Components outside the
set receive no new recovery authority, although ordinary parent-phase or
support events may still advance or invalidate them under existing bindings.
\end{proof}

Together, Theorems~\ref{thm:no-prior-revival}
and~\ref{thm:recovery-locality} establish the main-paper proposition on fresh
scoped recovery.

\paragraph{Protocol-conformance cohorts.}
The fixed admission oracle contains 96 predefined outcomes spanning local
support, malformed references, dependency removal, and maximal coherent
retention. The graph oracle adds 135 exhaustively enumerated graphs with
4, 8, or 16 components and 45 independently checked graphs with 32
components. A separate protocol-transition suite contains 20 hand-authored transition
cases and 10,000 seeded valid or fault-injected schedules covering event
ordering, final rechecks, token nonreuse, terminal parent failure, and scoped
recovery. The implementation matched every expected fixed outcome, graph
oracle, protocol-transition case, and generated schedule. These are deterministic
protocol-conformance checks rather than estimates of field failure
probability.

\begin{table*}[t]
\centering
\small
\setlength{\tabcolsep}{3pt}
\begin{tabular*}{\textwidth}{@{\extracolsep{\fill}}lll@{}}
\toprule
Main-paper question & Primary comparison & Evidence boundary \\
\midrule
RQ1: mechanism necessity & reduced variants and AgentSpec controls & component trace and PX4 setpoint \\
RQ2: protocol conformance & fixed cases, graph oracles, and schedules & protocol decisions \\
RQ3: fresh scoped recovery & recovery oracle, profiles, and UAV operation & new binding and sequence \\
RQ4: runtime cost & graph and gate size sweeps & protocol-only computation \\
Supplementary transfer & six repository variants & sandboxed merge ancestry \\
\bottomrule
\end{tabular*}
\caption{Mapping from evaluation questions to comparisons and observation
boundaries.}
\label{tab:supp-claim-map}
\end{table*}

\section{Evaluation Details}

\subsection{Live Structured-Agent Stress Check}
This cohort is a repetition-based implementation stress check, not a response-
diversity or generalization benchmark. It contains 100 OpenClaw Server-Sent
Events responses for one controlled four-component template. The collector saves a
terminal event, response provenance digest, parsed envelope, completion ledger,
and paired-replay linkage without storing the gateway token. Every accepted
response is replayed unchanged over 12 conditions and four variants, producing
4,800 paired rows.
Valid kept counts supported components, Stale dispatched counts stale action
effects, and Dependency leak counts dependency-inconsistent host-runtime notices.

\begin{table}[t]
\centering
\small
\setlength{\tabcolsep}{2pt}
\begin{tabular*}{\columnwidth}{@{\extracolsep{\fill}}lrrr@{}}
\toprule
Variant & Valid kept & Stale dispatch & Dep. leak \\
\midrule
WholeResponse & 400/2400 & 0/1000 & 0/1000 \\
IndependentFilter & 2400/2400 & 0/1000 & 1000/1000 \\
HALO-Admission & 2400/2400 & 500/1000 & 0/1000 \\
\textbf{HALO-Full} & \textbf{2400/2400} & \textbf{0/1000} & \textbf{0/1000} \\
\bottomrule
\end{tabular*}
\caption{Repeated fixed-template stress check. One hundred response
repetitions share a controlled structure so that paired mechanism effects can
be measured directly.}
\label{tab:supp-live-agent}
\end{table}

\subsection{Structured-Response Cohorts}
The diversity cohort is separate from the fixed-template cohort above. It uses
300 live OpenClaw calls across 15 frozen response-family schemas. The
collection funnel is \(300\rightarrow274\) schema-valid \(\rightarrow226\)
replay-eligible \(\rightarrow7{,}910\) paired replay rows. All accepted records
have validated provenance and unique endpoint response identifiers. The 74
rejected records comprise 28 family-schema mismatches, 20 missing-dependency
schema failures, 20 dependency-cycle failures, and six unknown-type schema
failures. They remain in the failure ledger and do not enter replay. This
cohort measures structural variation within one controlled endpoint and model,
not naturally occurring response distributions or cross-model generalization.

The endpoint reported the model identifier \texttt{openclaw}; it did not expose
a more specific model revision. Each of the 15 families received 20 calls using
its frozen prompt-state template and response schema. Temperature was omitted
from the request, the endpoint exposed no sampling-seed field, and the collector
used a 90-second timeout with at most two transient retries and a one-second
backoff. The separate 62-response AgentSpec cohort contains every
provenance-valid, schema-valid, family-schema-conforming record from its
frozen source collection; no outcome-based subsampling was performed.

\begin{table}[t]
\centering
\small
\setlength{\tabcolsep}{2pt}
\begin{tabular*}{\columnwidth}{@{\extracolsep{\fill}}P{0.29\columnwidth}P{0.18\columnwidth}rP{0.29\columnwidth}@{}}
\toprule
Cohort & Input & Eligible & Evaluation unit \\
\midrule
Response diversity & 300 calls & 226 & 7,910 paired policy evals. \\
AgentSpec comparison & 62 responses & 62 & 248 supported components \\
Independent-valid subset & 62 responses & 62 & 128 components \\
\bottomrule
\end{tabular*}
\caption{Structured-response cohort accounting.}
\label{tab:supp-cohort-accounting}
\end{table}

\begin{table}[t]
\centering
\footnotesize
\setlength{\tabcolsep}{2pt}
\begin{tabular*}{\columnwidth}{@{\extracolsep{\fill}}P{0.32\columnwidth}P{0.25\columnwidth}P{0.33\columnwidth}@{}}
\toprule
Variant & Retained mechanism & Removed or changed mechanism \\
\midrule
WholeResponse & bundle decision & partial preservation \\
AgentSpec-Global & global state rule & local support footprint \\
AgentSpec-Scoped & post-mutation scoped rules & no native dispatch-process test \\
GlobalSnapshot & proposal snapshot & local support footprint \\
IndependentFilter & local support & dependency closure \\
Local-Unbound & current predicates & exact admission binding \\
HALO-Admission & component admission & final dispatch gate \\
\textbf{HALO-Full} & \textbf{all HALO mechanisms} & \textbf{none} \\
\bottomrule
\end{tabular*}
\caption{Baseline definitions.}
\label{tab:supp-variants}
\end{table}

The late-policy replay freezes the baseline definitions before reading results.
GlobalSnapshot checks a proposal-wide current snapshot. Local-Unbound checks
the component-local freshness footprint but omits digest-bound epoch and bounds
provenance. The two are separated because their expected failure modes differ.
IndependentFilter evaluates current component-local support at the same
decision boundary as HALO but ignores declared dependencies.
AgentSpec-Scoped evaluates explicitly encoded component predicates at the
post-mutation decision point; it measures policy expressiveness rather than
AgentSpec's native dispatch process.
AgentSpec can therefore express equivalent admission predicates when they are
configured explicitly. HALO instead defines a coupled component-level protocol
that connects dependency-consistent admission, exact-instance dispatch
authority, late revalidation, and scoped fresh readmission across one
effect-bearing response.

\begin{table}[t]
\centering
\footnotesize
\setlength{\tabcolsep}{0.25pt}
\begin{tabular*}{\columnwidth}{@{\extracolsep{\fill}}lrrrr@{}}
\toprule
Method & Kept & Epoch & Bounds & Orphan \\
\midrule
AgentSpec-Global & 0/248 & 0/72 & 0/72 & 0/70 \\
AgentSpec-Scoped & 248/248 & 0/72 & 0/72 & 0/70 \\
IndependentFilter & 248/248 & 0/72 & 0/72 & 70/70 \\
HALO-Admission & 248/248 & 72/72 & 72/72 & 0/70 \\
\textbf{HALO-Full} & \textbf{248/248} & \textbf{0/72} & \textbf{0/72} & \textbf{0/70} \\
\bottomrule
\end{tabular*}
\caption{Structured-response comparison on 62 live envelopes. AgentSpec-Scoped
is the matched admission-time control, isolating initial predicate
expressiveness from HALO's later binding, gate, and recovery process.}
\label{tab:supp-generic}
\end{table}

These controlled prompts and schemas test structural variation across 15
response families while holding the model endpoint and replay conditions
fixed.

\subsection{Final-Gate and Authorization Controls}
The protected dispatch adapter uses one process-local critical section:
read the current component projection, validate applicable predicates, consume
a one-dispatch token, and enter the sole enqueue adapter. In six deterministic
cases, a post-validation mutation produced zero stale enqueues; only one of two
concurrent consumers could enqueue; used-token, restart-token, and
duplicate-sequence replay were rejected. A failed enqueue consumes its token
and requires fresh admission. This supports the gate-owned critical-section
behavior described in the main paper, not distributed atomicity or remote
exactly-once delivery.

The AgentSpec reference-interpreter comparison contains six scenarios with ten
warm repetitions. Policy-only AgentSpec produced 40/40 post-decision stale
internal fingerprints and allowed 10/10 token replays. Adding HALO's final
binding, component-bound one-dispatch token enforcement, and fresh recovery
protocol produced zero stale fingerprints, blocked all ten replay attempts,
and completed 20/20 fresh recoveries, matching HALO-Full at the measured PX4
boundary.

Each token is also bound to a component-local admission generation. Fresh
admission accepts only a failed or invalidated component instance, rejects every prior
sequence identity, revokes outstanding tokens from the previous generation,
and increments that generation. The gate checks the current component
instance and generation before token consumption. Consequently, obligation
readiness cannot revive the prior authorization: any recovered dispatch must
use current evidence, a later admission generation, a fresh token, and an
unused sequence.

For each recoverably withheld action-bearing component, the adapter can create
\begin{equation}
\label{eq:supp-obligation-record}
o_e=(e,r,S,Q,a,\rho,\tau,d),
\end{equation}
where \(r\) is the reason, \(S\) the required support, \(Q=S(o_e)\) the replacement
scope, \(a\) the assignee, \(\rho\) the route, \(\tau\) the expiry, and \(d\)
the discharge condition.
\(\operatorname{Exists}(o_e)\not\Rightarrow\operatorname{Dispatch}(e)\): an
obligation does not authorize the withheld effect.

Retry scheduling is catalog-defined. HALO provides the obligation state,
expiry, timeout and terminal exits, and the hook through which a candidate may
be resubmitted; every retry remains inside \(S(o_e)\) and repeats ordinary
admission with a new one-dispatch token and unused sequence.

\subsection{Cross-Domain Portability Check}
This check asks whether HALO can be instantiated at a versioned software
boundary without losing the tested preservation, blocking, and fresh recovery
behavior. We attach the unchanged core admission, binding, final-gate, and
recovery logic to a sandboxed bare-Git merge adapter. Git already supplies versioning, ancestry,
and merge semantics; HALO's role is only to decide whether an agent-returned
merge component, its dependent notice, and any recovered candidate may reach
that adapter under current external evidence and authorization.

The repository panel
contains 12 predefined cases, six protocol variants, and ten repetitions,
yielding 720 executions. Cases cover all-valid operation, unrelated branch or
issue updates, changed target head or patch digest, changed or failed
continuous-integration evidence, revoked approval, changed branch protection,
merge conflict, a dependent completion notice after a blocked merge, and fresh
recovery. Stale target-head, CI, approval, branch-protection, and patch-digest
cases withhold the merge candidate and its dependent
\texttt{merged\_notice}; catalog routes request rebase or replanning, fresh CI,
new approval, protection refresh, or patch-digest validation. Across the
complete cohort, Full HALO retained 40/40 unaffected outputs, produced 0 stale
merge observations and 0 dependent-notice leaks, and completed 10/10 fresh
readmissions. The reduced variants exposed the intended failures:
Local-Unbound produced 40 stale merges, HALO-Admission produced 90,
No-Dependency produced 90 dependent-notice leaks, and GlobalSnapshot falsely
denied 20/30 fresh controls. Git ancestry provides the downstream observation
of which merge commit, if any, was created. This is a sandboxed bare-Git
effect-boundary result, not a hosted-forge or distributed-transaction claim.

The effect-boundary catalog maps each represented type to its first protected
downstream effect: operator queue insertion, evidence-provider dispatch, handoff
enqueue, sandboxed merge enqueue, PX4-bound adapter enqueue, or status-channel
enqueue. A seeded 10,000-schedule property panel found zero tested locality,
local-validity, dependency-closure, greatest-set, token-nonreuse, or
blocked-parent-emission violations. It is implementation validation, not a
failure-probability estimate.

\subsection{PX4/Gazebo Observation Boundary}
The live \pxfour{}/\gazebo{} baseline panel \cite{meier2015px4,koenig2004gazebo} has 12 conditions and four
variants, for 480 case--variant executions across ten complete cold-start
SITL sessions in one fixed environment. A
predeclared NED-position-plus-yaw fingerprint links a component trace to an
internal \texttt{trajectory\_setpoint} observation. \pxfour{} does not carry a
native \method{} case identifier. Independent notices and dependency-inconsistent
notices are therefore evaluated at the HALO component-runtime boundary; only
action-bearing segments use the internal topic.

The late-drift experiment uses the following order:
\begin{equation}
\label{eq:supp-race-timeline}
\begin{aligned}
t_a\!:\ \text{admission}&\longrightarrow t_m\!:\ \text{mutation},\\
t_m&\longrightarrow t_c\!:\ \text{final gate},\\
t_c&\longrightarrow t_p\!:\ \text{host enqueue},\\
t_p&\longrightarrow t_x\!:\ \text{ULog observation}.
\end{aligned}
\end{equation}
The experiment requires $t_a<t_m<t_c$ for the injected epoch and bounds
mutations. Host timestamps $(t_a,t_m,t_c,t_p)$ and PX4 boot-time $t_x$ remain
different clocks. Alignment is used only to attribute broad process phases;
the internal fingerprint itself is matched directly in ULog coordinates.

\begin{table}[t]
\centering
\footnotesize
\setlength{\tabcolsep}{0.25pt}
\begin{tabular*}{\columnwidth}{@{\extracolsep{\fill}}lrrrr@{}}
\toprule
Variant & Kept & Stale & Dep. & Valid block \\
\midrule
WholeResponse & 40/240 & 0/50 & 0/100 & 200/200 \\
IndependentFilter & 240/240 & 0/50 & 100/100 & 0/200 \\
HALO-Admission & 240/240 & 50/50 & 0/100 & 0/200 \\
\textbf{HALO-Full} & \textbf{240/240} & \textbf{0/50} & \textbf{0/100} & \textbf{0/200} \\
\bottomrule
\end{tabular*}
\caption{PX4 results across ten cold starts. Stale dispatch counts
fingerprint-matched internal setpoints; dependency leaks are host notices.}
\label{tab:supp-px4-baseline}
\end{table}

The analysis observed all 170 expected internal fingerprints and no matching
fingerprint for any of 350 denied candidates. The 20 Full-\method{} recoveries
were observed only with fresh sequence identifiers. Hold intervals use direct
PX4 target fingerprints; execution and recovery use coarse host--PX4 process-phase
alignment, whose maximum residual was 108.04 ms. The alignment locates broad
process phases, while the internal setpoint evidence is matched directly in
PX4 ULog coordinates.

Across the ten cold sessions, required ULog topics and selected structured
fields were present. All 350 denied hold acquisitions and their steady
observations met the preregistered error and drift limits, and no selected true
status flag was attributed to an execution, deny-hold, or recovery interval.
Status samples outside those intervals remain separately reported in the
artifact rather than being counted as runtime-clean evidence.

\paragraph{UAV support-profile replication.}
A separate panel repeated 11 AoI conditions and 30 FAA Part~107-informed
engineering-profile conditions across ten cold-start sessions. The 41
unique conditions yielded 410 replays with 410/410 expected outcomes, no
matching internal setpoint for 250 denied actions, matching internal setpoints
for all 160 admitted actions, and fresh dispatch sequences for all 60
recoveries. The AoI subset contributed 110 replays, 0/50 denied setpoints,
60/60 admitted setpoints, and 10/10 fresh sequences; the FAA subset contributed
300 replays, 0/200 denied setpoints, 100/100 admitted setpoints, and 50/50
fresh sequences.

\paragraph{Physical Crazyflie operation.}
The main-paper Crazyflie~2.1 sequence instantiates the same runtime process with a
forward action and an independent downward-range report. Across ten repeated
operations, HALO withheld 10/10 stale forward actions, emitted 10/10 range
reports, observed 0/10 matching stale setpoints, and returned 10/10 recovery
candidates through fresh readmission. The representative trace shows initial
retention, a support change before the governed effect, selective withholding
of the action while the report remains available, and a fresh
backtrack-and-land candidate. The Crazyflie photographs visualize this sequence; the
formal evidence remains the component trace, adapter decision, and matching
setpoint check.

\subsection{Constrained-Motion Evidence}
The constrained-motion panel uses 14 scenarios with 20 warm repetitions in one
fixed \pxfour{}/\gazebo{} configuration. It records 200 denied and 80 admitted
or recovered segments. There were zero denied adapter calls, zero denied
matching target-telemetry observations, 40/40 fresh recoveries, 40/40 detected
post-final epoch or bounds mutations, and 280/280 oracle coverage. These warm
repetitions characterize timing and configuration stability under the fixed
PX4/Gazebo configuration.

The separate post-final race panel removes one binding mechanism at a time.
Within its tested epoch and bounds mutations, removing the final check, epoch
binding, or bounds binding selectively exposes the corresponding stale
component internally, while Full \method{} blocks it. This is controlled
ablation evidence at the measured boundary; it is not a general formal proof.

\begin{table}[t]
\centering
\footnotesize
\setlength{\tabcolsep}{1.5pt}
\begin{tabular*}{\columnwidth}{@{\extracolsep{\fill}}P{0.36\columnwidth}rrrr@{}}
\toprule
Measure & $n$ & Median (ms) & p95 (ms) & Max (ms) \\
\midrule
Validation & 280 & 0.325 & 0.491 & 23.205 \\
Final recheck & 280 & 0.005 & 0.010 & 0.030 \\
Gate to adapter & 280 & 0.005 & 0.011 & 0.031 \\
Adapter loop & 80 & 1577.832 & 1747.882 & 1764.751 \\
Motion completion & 80 & 1577.811 & 1747.859 & 1764.739 \\
Target activation & 79 & 1039.377 & 1155.511 & 1255.655 \\
Fresh recovery & 40 & 1649.957 & 1763.999 & 1765.188 \\
\bottomrule
\end{tabular*}
\caption{Constrained-motion latency decomposition using the floor-index
quantile estimator. Long adapter and completion intervals include setpoint
transmission, fingerprint dwell where applicable, host polling, telemetry,
PX4/\gazebo{} response, and target convergence; they are not HALO protocol
computation overhead.}
\label{tab:supp-latency}
\end{table}

\subsection{Recovery Cost}
The fixed recovery panel contains 48 routes. Thirty-six routes are executable
after refreshed support, and 12 are terminal by catalog policy. Table
\ref{tab:supp-recovery} reports the final-state recovery oracle. Old-token,
prior-sequence, and out-of-scope candidate rejection are tested separately in
the deterministic protocol suite. These controls rejected five
prior-authorization reuse attempts and one out-of-scope recovery candidate.
\begin{table*}[t]
\centering
\footnotesize
\setlength{\tabcolsep}{2.2pt}
\renewcommand{\arraystretch}{1.02}
\begin{tabular*}{\textwidth}{@{\extracolsep{\fill}}lcrrrrrr@{}}
\toprule
Measured interval & Layer & $n$ & Min & Median & Mean & p95 & Max \\
\midrule
Response validation & Protocol & 1,212 & 0.0018 & 0.0063 & 0.0105 & 0.0349 & 0.0677 \\
Local-support aggregation & Protocol & 1,212 & 0.0004 & 0.0010 & 0.0012 & 0.0036 & 0.0061 \\
Dependency closure & Protocol & 1,212 & 0.0059 & 0.0203 & 0.0266 & 0.0716 & 0.1267 \\
Retained-set materialization & Protocol & 1,212 & 0.0003 & 0.0008 & 0.0008 & 0.0015 & 0.0037 \\
Binding-record creation & Protocol & 1,212 & 0.0209 & 0.0850 & 0.0922 & 0.2063 & 0.3437 \\
Admission total & Protocol & 1,212 & 0.0348 & 0.1242 & 0.1448 & 0.3381 & 0.5344 \\
Fresh readmission total & Protocol & 1,212 & 0.0343 & 0.1233 & 0.1448 & 0.3415 & 0.5838 \\
Final gate, pre-token & Protocol & 10,836 & 0.0038 & 0.0123 & 0.0157 & 0.0352 & 0.5967 \\
Domain-predicate evaluation & Domain & 1,212 & 0.0016 & 0.0059 & 0.0062 & 0.0125 & 0.0295 \\
\bottomrule
\end{tabular*}
\caption{Protocol and domain-predicate compute in milliseconds. Values use the
nearest-rank p95 and retain all measured samples. Model generation, candidate
production, token mutation, adapter callbacks, PX4 activation and motion,
polling, and downstream transport are outside these intervals. The
reported values therefore isolate the protocol and domain-predicate layers.}
\label{tab:supp-protocol-compute}
\end{table*}

\begin{table*}[t!]
\centering
\small
\setlength{\tabcolsep}{4pt}
\begin{tabular*}{\textwidth}{@{\extracolsep{\fill}}P{0.19\textwidth}P{0.34\textwidth}P{0.40\textwidth}@{}}
\toprule
Boundary & Trusted or enforced & Claim boundary \\
\midrule
Raw structured proposal
& Untrusted; it has no retention or dispatch authority
& Free-form semantic decomposition or intent recovery \\
Schema, canonicalization, catalog, and dependency compiler
& Trusted computing base; malformed or mismatched inputs fail closed
& Completeness of types, payload detection, or generic dependency compilation \\
Certificates, predicates, epochs, and bounds
& Trusted inputs checked against the current component footprint
& Truth, availability, or completeness of providers \\
Token ledger and final gate
& Component-bound one-dispatch token enforcement at one process-local gate
& Distributed transaction atomicity or exactly-once PX4 execution \\
PX4 ULog fingerprints and clocks
& Measured internal-topic observation; direct hold markers and coarse process-phase alignment
& Full PX4 queue coverage or a shared host/PX4 clock \\
Controller and physical world
& Downstream of governed adapter dispatch
& Controller correctness, collision avoidance, or physical-flight safety \\
\bottomrule
\end{tabular*}
\caption{Trusted inputs, enforced boundaries, and the corresponding scope of
HALO's claims.}
\label{tab:supp-trust-boundary}
\end{table*}
\begin{table}[t]
\centering
\small
\begin{tabular*}{\columnwidth}{@{\extracolsep{\fill}}lr@{}}
\toprule
Outcome & Result \\
\midrule
Final-state oracle agreement & 48/48 \\
Fresh executable recovery & 36/36 \\
Terminal safe no-op & 12/12 \\
Unauthorized dispatch & 0/48 \\
Stale executable recovery & 0/36 \\
\bottomrule
\end{tabular*}
\caption{Fresh component-local recovery outcomes.}
\label{tab:supp-recovery}
\end{table}
\subsection{Protocol Computational Cost}

The computational-cost panel separates HALO protocol work from
domain-specific predicate evaluation and downstream transport or execution.
Table~\ref{tab:supp-protocol-compute} reports pooled descriptive samples across
the tested size configurations. Admission total begins at response validation
and ends after binding-record creation, using predicate results computed before
the timed interval. Fresh readmission begins after candidate production and
ends after creation of the later-generation binding. Final-gate compute begins
at gate entry and ends immediately before token consumption; token mutation and
the adapter callback are excluded.

All new intervals use \texttt{time.perf\_counter\_ns} with a reported
one-nanosecond clock resolution. Admission and readmission values are per-call
averages from batches of 50 calls after 20 explicit warm-up calls; garbage
collection is disabled only within timed batches. Gate configurations use 301
direct timed calls after 50 warm-up calls. These measurements are steady-state
Python-process measurements and do not mix PX4 cold starts. No benchmark
invariant or execution failures occurred, and all measured samples, including
the maximum gate observation, are retained.

The size sweep varied $|V|\in\{4,8,16,32\}$ and graph density, reaching
$|D_\phi|=200$ at $|V|=32$. The largest configuration-level admission median
was 0.333\,ms, with a largest p95 of 0.453\,ms. A median regression of the form
\[
\operatorname{median}_{\rm adm}
=\alpha+\beta_V|V|+\beta_E|D_\phi|
\]
yielded $R^2=0.9978$. The gate sweep varied
$|D_\phi(e)|\in\{0,1,2,4,8,16\}$ and
$|\pi_e(C_t)|\in\{1,2,4,8,16,32\}$. Its largest median pre-consumption compute
was 0.035\,ms and largest p95 was 0.061\,ms; the corresponding two-variable
median regression yielded $R^2=0.9975$. The observed scaling closely tracks
the source-level bounds
$O(|V|+|D_\phi|)$ for admission after predicate evaluation and
$O(|D_\phi(e)|+|\pi_e(C_t)|)$ for the final gate.

Table~\ref{tab:supp-protocol-compute} reports the shared core timing panel.
The repository check supplies transfer evidence, while the 280-run PX4
validation interval measures domain tube validation. Adapter activation,
motion, and recovery-total intervals are reported separately as downstream
system behavior.

\section{Reproduction and Validity}

\subsection{Artifact Reproduction}
The submitted artifact contains raw JSONL/CSV traces, ULog analyses, JUnit XML,
console logs, exact commands, environment records, and SHA-256 manifests. The
claim-evidence matrix records each claim's evidence layer, observation boundary,
sample count, supported wording, and scope.

The release launcher captures the base commit, porcelain status, binary diffs,
submodules, untracked-file hashes, build settings, binary hash, logger profile,
and launch allowlist before build and launch. Manifest-complete runs anchor the
reproduction package; earlier session-specific traces remain available as
auxiliary evidence.

\subsection{Threats to Validity}

The raw structured proposal is untrusted; the schema, canonicalization, catalog,
and declared-dependency compiler are trusted computing-base components.
The process-local token ledger and final gate define the enforced dispatch
boundary. The FAA-anchored profiles evaluate the represented operational
predicates, and PX4 status analyses use the preregistered runtime intervals
associated with each component trace.

\clearpage
\bibliography{aaai2027}

\end{document}